\documentclass[10pt, conference, compsocconf]{IEEEtran}
\IEEEoverridecommandlockouts
\ifCLASSINFOpdf
  \usepackage[pdftex]{graphicx}
  \graphicspath{{./fig/}}
  \DeclareGraphicsExtensions{.pdf,.jpeg,.png}
\else
\fi

\usepackage{fancyhdr,graphicx}
\usepackage[ruled,vlined,linesnumbered]{algorithm2e}

\usepackage{xspace}

\usepackage[utf8]{inputenc}

\usepackage{amssymb}
\usepackage{amsmath, amsthm}
\usepackage{mathabx}
\usepackage{cite}

\newcommand{\IN}{\mathbb{N}}
\newcommand{\IR}{\mathbb{R}}
\renewcommand{\AA}{\mathcal{A}}

\newcommand{\lbsymb}{\dlsh}

\newcommand{\A}{\Sigma}
\newcommand{\AStar}{\A^*}
\newcommand{\f}[1]{\textbf{#1}}

\newcommand{\fh}{\textbf{h}}
\newcommand{\fg}{\textbf{g}}
\newcommand{\fp}{\textbf{p}}
\newcommand{\fq}{\textbf{q}}

\newcommand{\HH}{\mathcal{H}}
\newcommand{\GG}{\mathcal{G}}

\newcommand{\NN}{\mathcal{N}}
\newcommand{\IP}{\mathbb{P}}

\newcommand{\IH}{\textbf{H}}
\newcommand{\IG}{\textbf{G}}
\newcommand{\FP}{\textbf{B}}

\newcommand{\D}[2]{\Delta_{#1,#2}}
\renewcommand{\d}[2]{\delta_{#1,#2}}

\newcommand{\brac}[1]{{\left(#1\right)}}
\newcommand{\bracabs}[1]{{\left|#1\right|}}
\newcommand{\bracround}[1]{{\left\{#1\right\}}}
\newcommand{\braceckig}[1]{{\left[#1\right]}}
\newcommand{\visiblespace}{\text{\textvisiblespace}}
\newcommand{\AExt}{\overline{\A}}
\newcommand{\AStarExt}{\AExt^*}
\newcommand{\fl}{\textbf{b}}
\newcommand{\bfl}{\overline{\fl}}
\newtheorem{theorem}{Theorem}
\newcommand{\VAA}{A}
\newcommand{\RAA}{{\AA^R}}
\newcommand{\GAA}{{\AA^G}}

\newcommand{\blind}{\boolean{false}}
\usepackage{setspace}
\usepackage{xcolor}	
\usepackage{ifthen}
\DeclareMathOperator*{\argmin}{arg\,min}

\newcommand*{\red}[1]{\textcolor{red}{#1}}	

\begin{document}
%
\title{End-To-End Measure for Text Recognition}

\author{
\ifthenelse{\blind}{\vspace{1.5cm}}{
	\IEEEauthorblockN{Gundram Leifert, Roger Labahn}
	\IEEEauthorblockA{Computational Intelligence Technology Lab\\
		University of Rostock\\
		18057 Rostock, Germany\\
		\{gundram.leifert,roger.labahn\}@uni-rostock.de
	}
	\and
	\IEEEauthorblockN{Tobias Grüning, Svenja Leifert}
	\IEEEauthorblockA{PLANET artificial intelligence GmbH\\
		Warnowufer 60\\
		18057 Rostock, Germany\\
		\{tobias.gruening,svenja.leifert\}@planet.de
	}
}
}
\maketitle

\begin{abstract}
Measuring the performance of text recognition and text line detection engines is an important step to objectively compare systems and their configuration. There exist well-established measures for both tasks separately. However, there is no sophisticated evaluation scheme to measure the quality of a combined text line detection and text recognition system. The F-measure on word level is a well-known methodology, which is sometimes used in this context. Nevertheless, it does not take into account the alignment of hypothesis and ground truth text and can lead to deceptive results. 
Since users of automatic information retrieval pipelines in the context of text recognition are mainly interested in the end-to-end performance of a given system, there is a strong need for such a measure. 
Hence, we present a measure to evaluate the quality of an end-to-end text recognition system. The basis for this measure is the well established and widely used character error rate, which is limited -- in its original form -- to aligned hypothesis and ground truth texts. The proposed measure is flexible in a way that it can be configured to penalize different reading orders between the hypothesis and ground truth and can take into account the geometric position of the text lines. Additionally, it can ignore over- and under- segmentation of text lines. With these parameters it is possible to get a measure fitting best to its own needs.
\end{abstract}
\IEEEpeerreviewmaketitle
\section{Introduction}
\label{S:intro}
Finding and reading textual information in an image is a common task in many real-world scenarios.
One application is the transcription of historical documents.
Typically, the focus is to transcribe the written text in the semantically correct order, whereas the geometric position of text lines is not in the scope of interest.
Another use case is to make a collection searchable, i.e., to allow for keyword spotting.
In such a scenario, a system is used to create some kind of index for the whole collection.
So the main focus is to find textual information in the image, whereas reading order of the text lines and sometimes even the text position is not of importance.
In contrast, there are other applications for which the geometric information of text lines is necessary, e.g. the postal inbox processing for insurances and banks.
Their purpose is to automatically read and classify all incoming letters.
Often, the input image should be enriched with a layer of textual information.
Therefore, geometric positions and the reading order of text lines are important to place the transcribed text at the right position.
Having these use cases with entirely different key aspects, there is the demand for a configurable end-to-end evaluation which is adaptable to the specific needs.
	
In the context of information retrieval the bag-of-word (BOW) measure is widely used \cite{bagofword2009}. It can be efficiently calculated by splitting the text into words and measuring precision, recall and F-measure of the text.
The BOW suffers from three major drawbacks. First, there is no unique definition of how a "word" should look like. This results in inconsistent and incomparable values of the BOW measure for different tokenizations of text lines into words. Second, a wrong character produces an error for the entire word. Comparably, segmentation errors are also penalized quite strongly. An erroneously recognized space character results in two word errors. Third, the BOW is not aware of any (potentially important) reading order and consequently does not penalize any permutation of recognized words.
	
For the decoupled problems of layout analysis (LA) and handwritten text recognition (HTR) there are well established measures. For the LA, which extracts text lines on pixel level, there are evaluation schemes based on different entities. For instance, based on pixel information \cite{marti2001using}, baselines \cite{DBLP:journals/corr/GruningLDKF17} or origin points \cite{Murdock2015}. Each of these schemes has its application area and consequently its right to exist.
On the other hand, the standard to evaluate the quality of an HTR system is the character error rate (CER), which has been used for decades. 
A major drawback of the CER is that it requires two aligned sequences of characters which usually are the transcriptions of text lines. This paper will provide task-dependent solutions for this alignment and an implementation is freely available supporting the well established PageXML format~\cite{pageXML}.

The paper is structured as follows:
Sec.~\ref{S:Measure} will derive the end-to-end CER from the classical CER and will motivate and define different configurations of this measure.
We will briefly demonstrate how to get from CER to word error rate (WER) and finally to BOW.
In Sec.~\ref{S:Implementation} the calculation of the introduced measures is described and the exactness of the proposed algorithms is proven for certain conditions.
A short summary and outlook concludes the paper in Sec.~\ref{S:Conclusion}.

\section{Measure Formulation}
\label{S:Measure}
The CER is based on the Levenshtein distance (LD), which counts the character manipulations (insertion, deletion, substitution) to map one string to another \cite{levenshtein1966binary}.
Let $\A$ be the alphabet of all characters and $\AStar$ the Kleene star of $\A$. 
Let $\fg_i\in\A$ be the i-th character of $\fg\in\AStar$ and $\fg_{i:j}:=\brac{\fg_i,\fg_{i+1},\ldots,\fg_j}$ a subsequence of $\fg$.
In the following it is required that the hypothesis (HYP) and ground truth (GT) $\fh,\fg\in\AStar$ do not have leading or trailing spaces\footnote{$\fh,\fg\in\AStar$ can be seen as sequence or tuple of characters, or as string}.
The LD between $\fh$ and $\fg$ is defined by recursion.
Let 
\begin{align}
\label{E:d_classic}
\d{i}{j}=\begin{cases}
0&\text{ if } \fh_i=\fg_j\\
1& \text{else}
\end{cases}
\end{align}
be the function that indicates the difference between $\fh_i$ and $\fg_j$. Let $\D{i}{j}=LD\brac{\fh_{1:i},\fg_{1:j}}$ be the number of manipulations which have to be done on $\fh_{1:i}$ to map it to $\fg_{1:j}$. This function is defined recursively over $i$ and $j$ with $\braceckig{n}:=\left\{1,...,n\right\}$ as follows 
\begin{align}
\D{0}{0}&=0\nonumber\;,\;\;
\D{i}{0}=i \quad \forall i\in\braceckig{\bracabs{\fh}}\nonumber\;,\;\;
\D{0}{j}=j \quad \forall j\in\braceckig{\bracabs{\fg}}\nonumber\,\\
\label{E:D_classic}
\D{i}{j}&=\min\bracround{\begin{array}{l}
\D{i-1}{j-1}+\d{i}{j}\\
\D{i-1}{j}+1\\
\D{i}{j-1}+1
\end{array}} \forall i\in\braceckig{\bracabs{\fh}},j\in\braceckig{\bracabs{\fg}},
\end{align}
so that we obtain the LD of the strings $\fh$ and $\fg$ by
\begin{align*}
LD\brac{\fh,\fg}:=LD\brac{\fh_{1:\bracabs{\fh}},\fg_{1:\bracabs{\fg}}}=\D{\bracabs{\fh}}{\bracabs{\fg}}.
\end{align*}
Since $\D{i}{j}$ in \eqref{E:D_classic} is recursively defined using values with one step back in $i$ or/and $j$, this problem can be efficiently solved using dynamic programming over the two-dimensional \mbox{$i$-$j$-space}.
Finally, the character error rate $CER:\AStar\times\AStar\to\IR^+$ is defined by
\begin{align*}
CER\brac{\fh,\fg}:=\frac{LD\brac{\fh,\fg}}{\bracabs{\fg}}.
\end{align*}
Of note, the CER could exceed $1$ and it is not commutative, i.e., $CER\brac{\fg,\fh}\neq CER\brac{\fh,\fg}$ for certain inputs $\fg,\fh$.

To evaluate a system's performance, the CER is calculated for a certain amount of text lines -- the so-called test set -- to get a reliable statistic. The test set is a $K$-tuple of GT sequences $\GG:=\brac{\GG_1,\ldots,\GG_K},\GG_k\in\AStar$. The HYP $\HH:=\brac{\HH_1,\ldots,\HH_K}$ is calculated by the system which has to be evaluated. The CER for a given test set is defined by
\begin{align*}
LD\brac{\HH,\GG}&:=\sum_{k=1}^K LD\brac{\HH_k,\GG_k}\\
\bracabs{\GG}&:=\sum_{k=1}^K\bracabs{\GG_k}\\
CER\brac{\HH,\GG}&:=\frac{LD\brac{\HH,\GG}}{\bracabs{\GG}}.
\end{align*}

To measure an end-to-end system, the CER calculation has to be extended from comparing two text lines to an arbitrary number of text lines of a page. For our proposed evaluation we expand the GT and HYP definition: Instead of a sequence of characters, we have a tuple of sequences of characters. For one fixed $k\in\braceckig{K}$ the $\HH_k,\GG_k\in\AStar$ become $\HH_k,\GG_k\in\brac{\AStar}^*$.
To calculate the CER the expansion of the denominator can be done straight-forward by
\begin{align*}
\bracabs{\GG}=\sum_{k=1}^K \bracabs{\GG_k}=\sum_{k=1}^K \sum_{x=1}^{\bracabs{\GG_k}}\bracabs{\brac{\GG_k}_x},
\end{align*}
whereas the expansion for the numerator
\begin{align*}
LD\brac{\HH,\GG}=\sum_{k=1}^K LD\brac{\HH_k,\GG_k}
\end{align*}
is non-trivial, because it is not clear how to calculate $LD\brac{\HH_k,\GG_k}$ easily.
Different ways to calculate
\begin{align*}
LD\brac{\IH,\IG}:=LD\brac{\HH_k,\GG_k}.
\end{align*}
will be proposed and discussed in the following.
$\IH,\IG\in\brac{\AStar}^*$ are tuples of character sequences, but $\bracabs{\IH}\neq\bracabs{\IG}$ has to be considered, which means that the numbers of text lines differ (mainly resulting from an erroneously working LA). 
The key idea is to expand \eqref{E:d_classic} and \eqref{E:D_classic} to match two tuples of character sequences.
Let $\IH:=\brac{\IH_1,\ldots,\IH_N}$ be the HYP lines and $\IG:=\brac{\IG_1,\ldots,\IG_M}$ the GT lines.
For the reason of simplicity, we write $\IH_y\in\IH$ if a text line belongs to the tuple of text lines, and $\IH'\subset\IH\Leftrightarrow\forall \IH_y'\in\IH':\IH_y'\in\IH$.
The \emph{Assignment Matrix} $A\in\AA$ defines which HYP and GT lines are assigned to each other.
We define the set of valid assignment matrices as 
\begin{align}
\label{E:A_vanilla}
\AA:=\bracround{\VAA\in\bracround{0,1}^{N\times M}\mid \left\|\VAA\right\|_1\leq1\land \left\|\VAA\right\|_\infty\leq1},
\end{align}
whereas $A_{y,x}=1$ means that $\IH_y$ and $\IG_x$ are assigned to each other.
The conditions in \eqref{E:A_vanilla} ensure that each GT line is assigned to at most one HYP line and vice versa.
With $A\in\AA$ it is possible to define the three sets
\begin{align*}
W&:=W\brac{A}=\bracround{\brac{y,x}\in\braceckig{N}\times\braceckig{{M}}\ \vert\ A_{y,x}=1},\\
U&:=U\brac{A}=\bracround{y\in\braceckig{N}\ \vert\ \forall x\in\braceckig{M}:A_{y,x}=0},\\
V&:=V\brac{A}=\bracround{x\in\braceckig{M}\ \vert\ \forall y\in\braceckig{N}:A_{y,x}=0},
\end{align*}
with $W$ containing the indices of the assigned text lines of $\IH$ and $\IG$ whereas $U$ and $V$ contain the indices of the unmatched text lines.
Note that all indices are in one of these sets, consequently $2\bracabs{W}+\bracabs{U}+\bracabs{V}=N+M$ holds.
The minimal LD is then defined by
\begin{align}
\label{E:LD-vanilla}
&LD\brac{\IH,\IG}:=LD_\AA\brac{\IH,\IG}=\\
&\min\limits_{A\in\AA}\sum_{\brac{y,x}\in W\brac{A}}LD\brac{\IH_y,\IG_x}+\sum_{y\in U\brac{A}} \bracabs{\IH_y}+\sum_{x\in V\brac{A}} \bracabs{\IG_x}.\nonumber
\end{align}
and the CER is defined by
\begin{align*}
CER\brac{\IH,\IG}=\frac{LD\brac{\IH,\IG}}{\bracabs{\IG}}.
\end{align*}
Of note, the LD in the sum of \eqref{E:LD-vanilla} is the basic LD which operates on single text lines.
If $CER\brac{\IH,\IG}=0$ holds, it is obvious that $\bracabs{\IG}=N=M=\bracabs{\IH}$, $A$ is a permutation matrix and $\forall A_{y,x}=1:\IH_y=\IG_x$. This also results in empty sets $U$ and $V$.

Next, we will describe different ways to modify this error rate. Whereas Sections \ref{S:measure:ro} and \ref{S:measure:geo} add restrictions for the LD calculation, Section \ref{S:measure:seg} allows a modification of $\IH$ to better match $\IG$. In Section \ref{S:measure:comb} we will discuss the combination of these modifications. Finally, a comparison between CER, WER and BOW is given in Section \ref{S:measure:cer-wer-bow}.
\subsection{Penalizing Reading Order Errors}
\label{S:measure:ro}
Even if the reading order of pages with tables, notes, marginalia or multiple columns is hard to define, it is crucial for semantic understanding.
So it is reasonable to extend the restriction of \eqref{E:A_vanilla} to
\begin{align*}
\RAA:=\big\{A\in\AA\ \mid\ &\forall y,y'\in\braceckig{N},\forall x,x'\in\braceckig{M}:\\&y<y'\land A_{y,x}=A_{y',x'}=1\Rightarrow x<x'\big\}.
\end{align*}
This additional restriction prevents assignments which are not aware of the orders of $\IH$ and $\IG$, e.g., an assignment for which the first line of $\IH$ is assigned to the last one of $\IG$ and vice versa.

We focus on the top right four text lines of Fig.~\ref{fig:sp} to demonstrate the effect for a simple example, i.e., 
\begin{align*}
\IH=\brac{\HH_4,\HH_7,\HH_5,\HH_8}\quad,\quad
\IG=\brac{\GG_7,\GG_8,\GG_{10},\GG_{11}}.
\end{align*}
In this order the HYP and GT only differ in the sorting along columns and rows as well as in one error in the hypothesis $\HH_7$.
Without the reading order constraint we get $LD_\AA\brac{\IH,\IG}=LD\brac{\HH_7,\GG_{10}}=1$, because $W=\bracround{\brac{1,1},\brac{3,2},\brac{2,3},\brac{4,4}}$ is feasible and there are no errors in three out of the four assigned text lines.
In contrast, with the constraint $A\in\RAA$ we get $W=\bracround{\brac{1,1},\brac{3,2},\brac{4,4}}$. The assignment $\brac{2,3}$ is not allowed due to the reading order constraint. Consequently,  
$\HH_7$ and $\GG_{10}$ are not assigned, which results in $U=\bracround{2}$, $V=\bracround{3}$ and $LD_\AA\brac{\IH,\IG}=\bracabs{\HH_7}+\bracabs{\GG_{10}}=5$, 

Based on \eqref{E:LD-vanilla} we define
\begin{align}
\label{E:LD-ro}
LD^R\brac{\IH,\IG}:=LD_\RAA\brac{\IH,\IG}
\end{align}
as minimal LD between $\IH$ and $\IG$ that penalizes reading order errors and $CER^R\brac{\IH,\IG}:=\frac{LD^R\brac{\IH,\IG}}{\bracabs{\IG}}$.
\subsection{Using Geometric Information as Restriction}
\label{S:measure:geo}
Especially for tables with short text lines containing for instance  the age, the birth date or running numbers, it is possible that the minimization of \eqref{E:LD-vanilla} assigns a wrongly transcribed HYP text line to a GT text line which is located at an entirely different position in the image.
E.g., $\HH_9=\GG_{10}$ holds for the text lines of Fig.~\ref{fig:sp}, but their geometric positions do not match. An assignment of such kind could erroneously reduce the CER. 
Consequently, it makes sense to only allow assignments between $\IH_y$ and $\IG_x$ if their geometric positions match.
Again, the idea is to add restrictions for $\AA$, such that two text lines can only be assigned, if they are ``(geometrically) close'' to each other.
There are many possibilities to determine if two text lines are close to each other or not.
Here, the well-established method of \cite{DBLP:journals/corr/GruningLDKF17} is used. We say two text lines are close, if their baselines are geometrically close to each other (see Section \ref{S:Implementation:geo} for details).
Let $\NN\brac{\IG_x}\subset\IH$ be the set of all text lines in $\IH$, that are close to $\IG_x$.
We extend \eqref{E:A_vanilla} to
\begin{align}
\label{E:A-geom}
\GAA:=\big\{A\in\AA\mid A_{y,x}=1\Rightarrow \IH_y\in\NN\brac{\IG_x}\}
\end{align}
and modify \eqref{E:LD-vanilla} with $\AA^G$ 
\begin{align}
\label{E:LD-geom}
LD^G\brac{\IH,\IG}:=LD_\GAA\brac{\IH,\IG}
\end{align}
to define $CER^G\brac{\IH,\IG}:=\frac{LD^G\brac{\IH,\IG}}{\bracabs{\IG}}$.
\begin{figure*}[!t]
	\setlength{\tabcolsep}{3px}
	\centering
	\begin{minipage}{.7\textwidth}
		\includegraphics[width=1.0\textwidth]{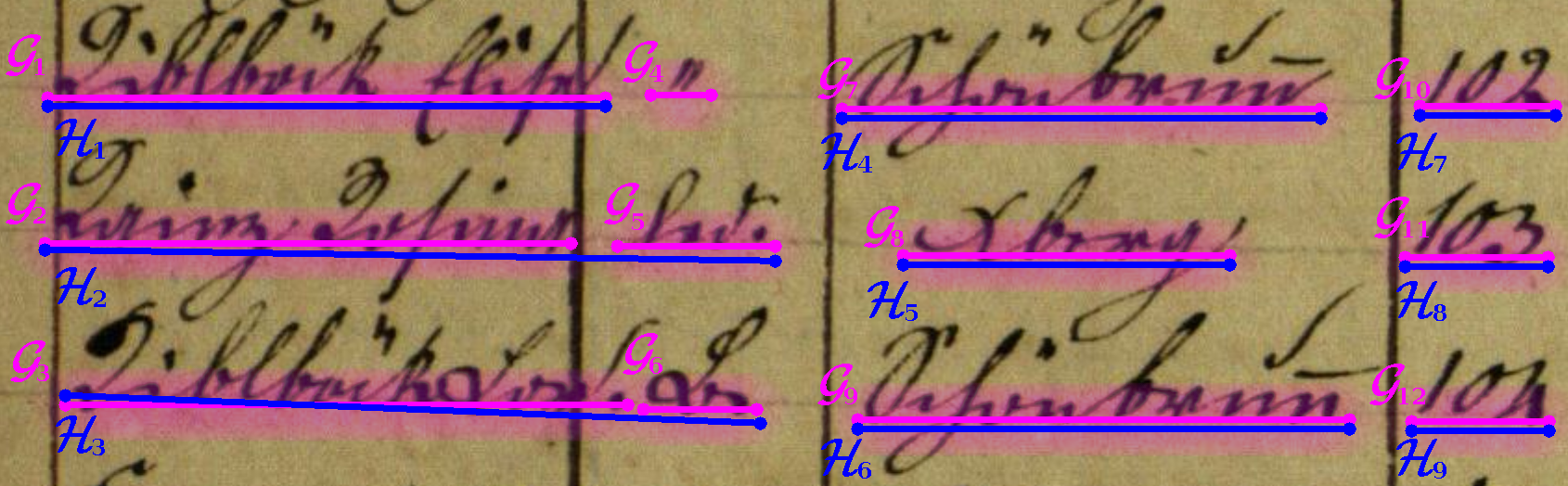}
	\end{minipage}
	\begin{minipage}{.28\textwidth}
		\smaller
		\begin{tabular}{|r|r|c|c|}
			\hline
			i & j & $\HH_i$ & $\GG_j$ \\
			\hline
			1 & 1 &Küblböck Elise&Küblböck Elise\\
			2 & 2 &Kainz Josina Led.&Kainz Josina\\
			3 & 3 &KüblböckLed. L.&Küblböck Led.\\
			& 4 & & " \\
			& 5 & & Led. \\
			& 6 & & L. \\
			4 & 7 & Schönbrunn & Schönbrunn \\
			5 & 8 & Aberg & Aberg \\
			6 & 9 & Schönbrunn & Schönbrunn \\
			7 &10 & 10 & 102 \\
			8 &11 & 103 & 103 \\
			9 &12 & 102 & 104 \\
			\hline 
		\end{tabular}
	\end{minipage}
	\caption{\textbf{Example of a table with column-wise sort of text lines.} Two common LA errors are missing baselines (see $\GG_4$) and erroneously merging of text lines (see $\brac{\GG_2,\GG_5}\leftrightarrow\HH_2$). Also the reading order can cause errors: In the HYP, the first two columns are merged together, so that the transcription "Led." of $\GG_5$ is ordered before $\HH_3$, but after $\GG_3$. Dependent on the configuration, these errors influence the measure (see Figure~\ref{fig:errortable}).
	}
	\label{fig:sp}
\end{figure*}
\begin{figure}[!h]
	\centering
	{
		\setlength{\tabcolsep}{3px}
		\begin{tabular}{|c|cccc|ccc|}
			\hline 
			\rule[-1ex]{0pt}{2.5ex}  & INS & DEL & SUB & COR & CER & Prec & Rec \\ 
			\hline 
			\rule[-1ex]{0pt}{2.5ex}  $CER^R$ & 9 & 8 & 1 & 70 & 22.5\% & 88.6\% & 88.1\% \\ 
			\hline 
			\rule[-1ex]{0pt}{2.5ex}  & INS & DEL & SUB & COR & WER & Prec & Rec \\ 
			\hline 
			\rule[-1ex]{0pt}{2.5ex}  $WER^R$ & 3 & 1 & 4 & 8 & 53.3\% & 61.5\% & 53.3\% \\ 
			\rule[-1ex]{0pt}{2.5ex}  $WER$ & 3 & 1 & 3 & 9 & 46.7\% & 69.2\% & 60.0\% \\ 
			\rule[-1ex]{0pt}{2.5ex}  $WER^G$ & 3 & 1 & 4 & 8 & 53.3\% & 61.5\% & 53.3\% \\ 
			\rule[-1ex]{0pt}{2.5ex}  $WER^S$ & 2 & 0 & 2 & 11 & 26.7\% & 84.6\% & 73.3\% \\ 
			\rule[-1ex]{0pt}{2.5ex}  $WER^{S,G}$ & 2 & 0 & 3 & 10 & 33.3\% & 76.9\% & 66.7\% \\ 
			\hline 
			\rule[-1ex]{0pt}{2.5ex}  & FN & FP & & TP & & Prec & Rec \\ 
			\hline 
			\rule[-1ex]{0pt}{2.5ex}  $BOW$ & 4 & 2 &  & 11 &  & 84.6\% & 73.3\% \\ 
			\rule[-1ex]{0pt}{2.5ex}  $BOW^G$ & 5 & 3 &  & 10 &  & 76.9\% & 66.7\% \\
			\hline 
		\end{tabular}
	}
	\caption{\label{fig:errortable}\textbf{Comparison of Measures.} The error rates are calculated from the transcripts and polygons shown in Figure~\ref{fig:sp}. For CER and WER we can define precison (Prec) and recall (Rec) similar to the measures of BOW (see \eqref{E:Prec},\eqref{E:Rec}). In this example $WER^S$ and $BOW$ result even in the same precision and recall values.
		Whereas $WER$ finds a correct assignment between $\HH_9$ and $\GG_{10}$, $WER^R$ and $WER^G$ avoid this either by the forced reading order or by the comparison of the corresponding baselines. If we allow segmentation errors, $WER^S$ correctly assigns $\HH_2$ to $\GG_2$ and $\GG_5$.  }
\end{figure}
\subsection{Non-Penalizing of Segmentation Errors}
\label{S:measure:seg}
If an LA does not detect a text line $\IG_x$, the LD increases by $\bracabs{\IG_x}$, as well as the LD increases by $\bracabs{\IH_y}$ for an erroneously detected text line $\IH_y$.
Even more crucial are falsely merged text lines.
For example, the hypothesis $\brac{\HH_2}$ of Figure~\ref{fig:sp} is an erroneously merged text line. For
\begin{align}
\label{E:G_H_seg}
\IH=\brac{\HH_2}\quad,\quad\IG=\brac{\GG_2,\GG_5},
\end{align}
the calculation  $LD\brac{\IH,\IG}$ leads to $U=\emptyset,V=\bracround{2},W=\bracround{\brac{1,1}}$ and $LD\brac{\IH,\IG}=LD\brac{\HH_2,\GG_2}+\bracabs{\GG_5}=5 + 4=9$. The resulting LD could be considered to be quite high based on the fact, that the recognized text is entirely correct, but merged.
The same argument is valid for an erroneous split of a text line.
Hence, it is meaningful to modify the LD calculation such that it does not penalizes this kind of split and merge errors.

It is assumed that these kind of segmentation errors are mainly caused by large gaps between words. As a result, the most common substitution for a line break is the space character $\visiblespace\in\A$ in the merged line. Hence, we allow to interpret a line break as space character and the other way around. This is achieved by allowing successive split operations at spaces and merge operations between lines to adjust $\IH$:
\begin{itemize}
	\item split operation: One line  $\fh=\IH_y$ with the space character $\fh_k=\visiblespace$ at position $k$ can be split into two lines $\f{a}=\brac{\fh_1,\ldots,\fh_{k-1}}$ and $\f{b}=\brac{\fh_{k+1},\ldots,\fh_{\bracabs{\fh}}}$,
	\item merge operation: Two subsequent lines $\f{a}=\IH_y$ and $\f{b}=\IH_{y+1}$ can be merged into one line $\brac{\f{a}_1,\ldots,\f{a}_{\bracabs{\f{a}}},\visiblespace,\f{b}_1,\ldots,\f{b}_{\bracabs{\f{b}}}}$.
\end{itemize}
We define the \emph{space of partition functions}
\begin{align}
\Psi:=&\left\{\Phi:\brac{\AStar}^*\to\brac{\AStar}^*\right\}
\end{align}
with $\Phi$ as a composition of split and merge operations.
We change \eqref{E:LD-vanilla} by optimizing over all $\Phi$ that minimizes the LD:
\begin{align}
\label{E:LD-seg}
LD^S\brac{\IH,\IG}:=\min_{\Phi\in\Psi}LD\brac{\Phi\brac{\IH},\IG}.
\end{align}
and get $CER^S\brac{\IH,\IG}:=\frac{LD^S\brac{\IH,\IG}}{\bracabs{\IG}}$. For the example in \eqref{E:G_H_seg} and the optimal $\Phi^*$ we get
\begin{align*}
\IH'=\Phi^*\brac{\IH}=\brac{\text{"Kainz Josina"},\text{"Led."}}
\end{align*}
which leads to $CER^S\brac{\IH,\IG}=CER\brac{\IH',\IG}=0$.

It has to be mentioned, that $\bracabs{\IH}\neq\bracabs{\Phi\brac{\IH}}$ is possible. Furthermore, for an optimal $\Phi^*$ there is no text line $\IH_y\in\Phi^*\brac{\IH}$ in $U$ which contains spaces, because a splitting of $\IH_y$ at this spaces would result in a lower LD. 
\subsection{Combination of Measure Modifications}
\label{S:measure:comb}
The equations \eqref{E:LD-ro}, \eqref{E:LD-geom} and \eqref{E:LD-seg} are defined as single modifications of \eqref{E:LD-vanilla}, whereby in many scenarios a combination of these modifications is reasonable:
For example to measure the quality of a text extraction method, the semantic meaning is important, which leads to the reading order restriction combined with the option to change the segmentation. We will denote combinations of configurations by adding all modification letters to the superscript (in the previous example: $CER^{R,S}$).
Having $3$ modifications we can choose between $2^3=8$ configuration-dependent CER measures. 

Besides the possibility to evaluate the quality of an HTR engine under different restrictions, a meaningful comparison of the results for different measure configurations allows for an examination of the categories of the main errors of this system.
\subsection{From CER over WER to BOW}
\label{S:measure:cer-wer-bow}
The WER can be determined based on the CER methodology introduced in this chapter. If $\A$ is not chosen as alphabet of characters but of words instead, everything in Section \ref{S:Measure} holds and the $CER$ becomes the $WER$. Hence, the $WER$ with all different configurations can be calculated.

There is no general definition of how to transform a sequence of character into a sequence of words. For example, the sequence ``it's'' could be divided into one, two or three words. Since in the most cases the user has his own idea of ``words'', we provide a simple interface to integrate own word tokenizers\footnote{Interface: https://github.com/Transkribus/TranskribusInterfaces/blob/ master/src/main/java/eu/transkribus/interfaces/ITokenizer.java}.
A basic tokenizer, that splits a character sequence at spaces, is implemented as default.
For Figure~\ref{fig:errortable} this tokenizer is used.

At first glance $CER$ does not have much in common with $BOW$. However, by successively changing the configurations, we can close the gap between these measures:
\begin{align*}
CER^R\leftrightarrow WER^R\leftrightarrow WER \leftrightarrow WER^S \leftrightarrow BOW
\end{align*}
So far, it is not obvious, why $WER^S \leftrightarrow BOW$ is reasonable.
For the $WER$ calculation we do not only count the manipulations insertion, deletion and substitution, we also count the number of correctly assigned characters/words ($COR$).
For the $BOW$ measure the false positive ($FP$), the false negative ($FN$) and the true positive ($TP$) words are counted.
So we can define precision and recall for $WER$ and $CER$ with similar counts used in $BOW$:
\begin{align}
\label{E:Prec}
Prec&:=\frac{COR}{\bracabs{HYP}}\quad\leq\frac{TP}{TP+FP}=\frac{TP}{\bracabs{HYP}} \\
\label{E:Rec}
Rec&:=\frac{COR}{\bracabs{GT}}\quad\leq\frac{TP}{TP+FN}= \frac{TP}{\bracabs{GT}},
\end{align}
whereas $\bracabs{GT}$ and $\bracabs{HPY}$ are the number of characters/words in GT and HYP.
Note that in Figure~\ref{fig:errortable} for precision and recall for $WER^S$ and $BOW$ are equal, even with additional geometric restrictions.
Since $WER^S$ is constructed to minimize the LD, which only implicitly maximizes $COR$, the inequality is obvious. 
But if the lines of $\IG$ or $\IH$ are single words, it follows equality and we closed the gap between $WER^S$ and $BOW$.

\section{Algorithm description}
\label{S:Implementation}
In this section the implementation details for four out of the eight possible measure configurations, namely $LD,\ LD^R,\ LD^{R,G},\ LD^{R,S}$, are described. Furthermore, it is discussed if the proposed algorithms result in the minimal LDs -- if they solve the minimization problems see \eqref{E:LD-vanilla}, \eqref{E:LD-ro}, \eqref{E:LD-geom} and \eqref{E:LD-seg} exactly -- or not.

Since the set of possible assignment matrices \eqref{E:A_vanilla} allows for arbitrary line permutations of $\IH_y,y\in\braceckig{N}$, its cardinal number exceeds the factorial number $N!$. Consequently, for practical relevant values of $N$ the calculation of $LD$ -- the optimization of \eqref{E:LD-vanilla} -- becomes intractable and will not be computed exactly. In Sec.~\ref{S:Implementation:RO} a greedy algorithm is introduced to find the (greedy-)optimal assignment matrix $A\in\AA$.
In the cases of $LD^R,\ LD^{R,G},\ LD^{R,S}$, the constraint of a fixed reading order, see Sec.~\ref{S:measure:ro}, allows for the formulation of exact algorithms. In Sec.~\ref{S:Implementation:vanilla} - \ref{S:Implementation:SEG} these algorithms are introduced and it is proven that they result in global minima for the LDs.

As shown in Section~\ref{S:Measure}, the LD can be calculated using dynamic programming over subsequences between  $\fh,\fg\in\AStar$ which leads to a two-dimensional calculation problem.
Because of $\IH,\IG\in\brac{\AStar}^*$, the dynamic programming becomes four-dimensional.
We avoid this by flattening $\IH,\IG$ to one dimension in a first step, such that the dynamic programming remains two-dimensional. 
Therefore, we add the artificial line break character $\lbsymb\ \notin\A$ to the alphabet and get $\AExt:=\A\cup\bracround{\lbsymb}$.
Let
\begin{align}
\label{E:f}
f:\brac{\AStar}^*\to\AStarExt
\end{align}
be the invertible \emph{flatten function} that concatenates the text line and puts $\lbsymb$ before, between and after the lines. For example we obtain
\begin{align*}
f\brac{\brac{a,b},\brac{c,d},\brac{e,f}}=\brac{\lbsymb,a,b,\lbsymb,c,d,\lbsymb,e,f,\lbsymb}.
\end{align*}
Finally, the flattened hypothesis and ground truth lines are defined as 
$\fh:=f\brac{\IH},\fg:=f\brac{\IG}$ with $\fh,\fg\in\AStarExt$.

In the next sections configure-dependent equations to calculate the LDs for the different restrictions are proposed.

\subsection{Exact Calculation of $LD^R\brac{\IH,\IG}$}
\label{S:Implementation:vanilla}
We use the recursion defined in \eqref{E:D_classic} and expand it to calculate the LD across text lines for the flattened $\fh,\fg$. For that purpose, we expand \eqref{E:d_classic} to
\begin{align}
\label{E:d_vanilla}
\d{i}{j}^R:=\begin{cases}
0 &\text{ if } \fh_i=\fg_j\\
1 &\text{ if } \fh_i\neq\fg_j \land \fh_i,\fg_j\in\A\\
\infty &\text {else}
\end{cases}.
\end{align}
This adaptation will prevent substitutions of usual characters by line break characters and vice versa.
Consequently, only line breaks can be mapped to each other. Hence, this enforces a direct comparison of entire text lines instead of parts of text lines.
Let $\fl_\fh\in \braceckig{\bracabs{\fh}}^{\bracabs{\IH}+1}$ be the tuple of line break positions in $\fh$, whereas $\fl_\fh^y:=\brac{\fl_\fh}_y\in\braceckig{\bracabs{\fh}}$ is the index of the $y$-th line break in $\fh$.
The tuple $\fl_\fg$ is defined in the same manner.
For simplification we use the notation of the cross product of sets for tuples:
\begin{align*}
\brac{i,j}\in\fl_\fh\times\fl_\fg&\Leftrightarrow i\in\fl_\fh\land j\in\fl_\fg.
\end{align*}
For index pairs $\brac{i,j}\in\fl_\fh\times\fl_\fg$ which represent line breaks at $i=\fl_\fh^y$ and $j=\fl_\fg^x$, we modify the distance calculation in \eqref{E:D_classic} to allow for the deletion and insertion of lines
\begin{align}
\label{E:D_vanilla}
\D{i}{j}^R=\min\left\{\begin{array}{ll}
\D{i-1}{j-1}^R\\
\D{\fl_\fh^{y-1}}{j}^R+\bracabs{\IH_{y-1}} &\text { if } y\geq 2\\
\D{i}{\fl_\fg^{x-1}}^R+\bracabs{\IG_{x-1}} &\text { if } x\geq 2
\end{array}\right\},
\end{align}
for other index pairs $\brac{i,j}\in \brac{ \braceckig{\bracabs{\fh}}\times \braceckig{\bracabs{\fg}}}\setminus\brac{\fl_\fh\times\fl_\fg}$ we set
\begin{align}
\label{E:D_vanilla2}
\D{i}{j}^R=\min\left\{\begin{array}{ll}
\D{i-1}{j-1}^R+\d{i}{j}^R\\
\D{i-1}{j}^R+1 &\text { if } \fh_i\neq\lbsymb\\
\D{i}{j-1}^R+1 &\text { if } \fg_j\neq\lbsymb
\end{array}\right\}.
\end{align}
In the following, we use the term \textit{points} for index pairs.


\begin{theorem}[Minimal $LD^R$ calculation]
\label{theo:LD-vanilla}
Let $\fh=f\brac{\IH}$ and $\fg=f\brac{\IG}$ be the flattened sequences. The following equality holds 
\begin{align}
\label{E:proof}
LD^R\brac{\IH,\IG}=\D{\bracabs{\fh}}{\bracabs{\fg}}^R.
\end{align}
\end{theorem}
\begin{proof}
If for each point $\brac{i,j}$ the minimal predecessor is stored and the final value $LD^R\brac{\fh,\fg}=\D{\bracabs{\fh}}{\bracabs{\fg}}^R$ is calculated, the path leading to the minimal LD can be recursively reconstructed, starting from point $\brac{\bracabs{\fh},\bracabs{\fg}}$ until ending in $\brac{0,0}$. 
\begin{align*}
P:=\brac{\brac{0,0},\ldots,\brac{\bracabs{\fh},\bracabs{\fg}}}\in\brac{\IN^2}^*
\end{align*}
the \emph{best path}. Due to \eqref{E:d_vanilla} and \eqref{E:D_vanilla} the path contains all line breaks of $\fh$ and $\fg$. 
As shown in Alg. \ref{alg:splitbestpath} $U,V$ and $W$ can be obtained from $P$.
We use induction over the number of accumulated lines in $\IH$ and $\IG$ (which is $K=\bracabs{\IH}+\bracabs{\IG}$), to show that \eqref{E:proof} holds.\\
For $K=0$ we have $\IH=\IG=\emptyset$ and $LD\brac{\IH,\IG}=0$.\\
For $K\geq1$ with $\IH=\emptyset,\fh=\brac{\lbsymb}$ and $\bracabs{\IG}=K\geq1$, \eqref{E:d_vanilla} and \eqref{E:D_vanilla} result into one single path
\begin{align*}
P=\brac{\brac{0,0},\brac{1,\fl_\fg^1},\ldots,\brac{1,\fl_\fg^{M+1}}}
\end{align*}
and we can calculate the LD
\begin{align*}
\D{1}{\fl_\fg^{1}}^R&=\D{1}{1}^R=\D{0}{0}^R+0=0\\
\D{1}{\fl_\fg^{x}}^R&=\D{1}{\fl_\fg^{x-1}}^R+\bracabs{\IG_{x-1}}=\sum_{j=1}^{x-1} \bracabs{\IG_j}\\
LD^R\brac{\fh,\fg}&=\D{\bracabs{\fh}}{\bracabs{\fg}}^R=\D{1}{\fl_\fg^{M+1}}^R=\sum_{j=1}^{M} \bracabs{\IG_j}.
\end{align*}
The same argument can be used for the calculation of $\bracabs{\IH}\geq1$ and $\IG=\emptyset$.

Now, we apply induction over $K$ for $\bracabs{\IH},\bracabs{\IG}\geq 1$.
Let $\IH':=\IH\setminus\bracround{\IH_{\bracabs{\IH}}}$ and $\IG':=\IG\setminus\bracround{\IG_{\bracabs{\IG}}}$ be tuples of text lines without the last text line.
As induction hypothesis we assume $LD^R\brac{\IH',\IG'}$ ($=K-2$), $LD^R\brac{\IH,\IG'}$ and $LD^R\brac{\IH',\IG}$ ($=K-1$) are correctly calculated.
We will show that we can calculate $LD^R\brac{\IH,\IG}=\D{\bracabs{\fh}}{\bracabs{\fg}}^R$ using the induction hypothesis.

Let $\fh':=f\brac{\IH'}$ and $\fg':=f\brac{\IG'}$ be the flattened HYP and GT.
Since $\forall i\in\bracabs{\fh'}:\fh'_i=\fh_i$ it follows \eqref{E:D_vanilla} will be the same no matter if we compare with $\IH$ or $\IH'$.
The same argument holds for $\IG$ and $\IG'$.

All paths ending in the point $\brac{{\bracabs{\fh}-1},{\bracabs{\fg}-1}}$ contain $\brac{{\fl_\fh^{\bracabs{\fh}-1}},{\fl_\fg^{\bracabs{\fg}-1}}}=\brac{{\fl_{\fh'}^{\bracabs{\fh'}}},{\fl_{\fg'}^{\bracabs{\fg'}}}}=\brac{{\bracabs{\fh'}},{\bracabs{\fg'}}}$.
So we separately calculate the LD for both parts, which is
\begin{align*}
\D{\bracabs{\fh}-1}{\bracabs{\fg}-1}^R=LD^R\brac{\IH',\IG'}+LD^R\brac{\IH_\bracabs{\IH},\IG_\bracabs{\IG}}.
\end{align*}
If we set $i=\bracabs{\fh}=\fl_{\fh}^{\bracabs{\IH}+1}$, $j=\bracabs{\fg}=\fl_{\fg}^{\bracabs{\IG}+1}$ in \eqref{E:D_vanilla} and use $\fl_{\fh}^{\bracabs{\IH}}=\fl_{\fh'}^{\bracabs{\IH'}+1}=\bracabs{\fh'}$ and $\fl_{\fg}^{\bracabs{\IG}}=\fl_{\fg'}^{\bracabs{\IG'}+1}=\bracabs{\fg'}$ we get
\begin{align*}
\D{\bracabs{\fh}}{\bracabs{\fg}}^R&=\min\left\{\begin{array}{l}
\D{\bracabs{\fh}-1}{\bracabs{\fg}-1}^R\\
\D{\fl_{\fh}^{\bracabs{\IH}}}{\bracabs{\fg}}^R+\bracabs{\IH_{\bracabs{\IH}}}\\
\D{\bracabs{\fh}}{\fl_{\fg}^{\bracabs{\IG}}}^R+\bracabs{\IG_{\bracabs{\IG}}}
\end{array}\right\}\\
&=\min\left\{\begin{array}{l}
LD^R\brac{\IH',\IG'}+LD^R\brac{\IH_\bracabs{\IH},\IG_\bracabs{\IG}}\\
LD^R\brac{\IH',\IG}+\bracabs{\IH_{\bracabs{\IH}}}\\
LD^R\brac{\IH,\IG'}+\bracabs{\IG_{\bracabs{\IG}}}
\end{array}\right\}
\end{align*}
Each row indicates how $U$, $V$ and $W$ are expanded over the recursion:
When the first row is the minimum this leads to $\brac{N,M}\in W$, whereby when the second (third) row is the minimum we have $N\in U$ (or $M\in V$).
So $LD\brac{\IH,\IG}=\D{\bracabs{\fh}}{\bracabs{\fg}}^R$ is the minimum of these three sub problems with additional costs as defined in \eqref{E:LD-vanilla}.
\end{proof}
\begin{algorithm}
	\small
	\RestyleAlgo{boxed}
	\DontPrintSemicolon
	\SetKwInOut{Input}{input}
	\SetKwInOut{Output}{output}
	\SetKwComment{Comment}{\%}{}
	\Input{$P,\fl_{\fh},\fl_\fg$}
	\Output{$U,V,W$}
	\BlankLine	\BlankLine
	$U,V,W\gets \emptyset$\;
	$\fp\gets P_2$\Comment*{$P_1=\brac{0,0}$ is not of interest}
	\For{$i=3,\ldots,\bracabs{P}$}{
		$\fq\gets P_i$\;
		\If{$\fq_2\in\fl_\fg$}{
			\Comment{Found line break in $\fg$ ($\fg_{\fq_2}=\lbsymb$)}
			$x\gets index\brac{\fl_\fg;\fp_2}$\Comment*{$x$-th $\lbsymb$ in $\fg$}
			$x'\gets index\brac{\fl_\fg;\fq_2}$\Comment*{$x'$-th $\lbsymb$ in $\fg$}
			$y\gets index\brac{\fl_\fh;\fp_1}$\Comment*{$y$-th $\lbsymb$ in $\fh$}
			$y'\gets index\brac{\fl_\fh;\fq_1}$\Comment*{$y'$-th $\lbsymb$ in $\fh$}
			\If{$y<y'$}{
				\If{$x<x'$}{
					$W\gets W\cup\bracround{\brac{y,x}}$\Comment*[r]{$\IH_y$ maps $\IG_x$}
				}\Else{
					$U\gets U\cup\bracround{y}$\Comment*[r]{delete $\IH_y$}
				}
			}\Else{
				$V\gets V\cup\bracround{x}$\Comment*[r]{delete $\IG_x$}
			}	
			$\fp\gets \fq$\Comment*{end point is the new start}
		}
	}
	\Return{$U,V,W$}
	\caption{\label{alg:splitbestpath}SplitBestPath}
\end{algorithm}

The calculation of $\D{\bracabs{\fh}}{\bracabs{\fg}}^R$ can be formulated as shortest path problem. Therefore, we search the shortest path from point $\brac{0,0}$ to $\brac{i,j}$, which indicates the minimal cost to map $\fh_{1:i}$ to $\fg_{1:j}$. For $\brac{i,j}=\brac{\bracabs{\fh},\bracabs{\fg}}$ we obtain $LD^R\brac{\IH,\IG}=\D{\bracabs{\fh}}{\bracabs{\fg}}^R$. Since at each point we calculate the minimum over other points with additional non-negative costs, we can use the \emph{Dijkstra Algorithm} to solve this problem \cite{Dijkstra1959}. Especially for a low CER this algorithm can skip the calculation of many points $\brac{i,j}\in \braceckig{\bracabs{\fh}}\times \braceckig{\bracabs{\fg}}$. The implementation is done in Java and freely available on GitHub\footnote{\ifthenelse{\blind}{BLIND REPO}{https://github.com/CITlabRostock/CITlabErrorRate}} under the Apache License.
\subsection{Restricting by Geometric Position}
\label{S:Implementation:geo}
As mentioned in Section \ref{S:measure:geo} it is reasonable to allow $\brac{y,x}\in W$, only if $\IH_y$, $\IG_x$ are geometrically close to each other ($\IH_y\in\NN\brac{\IG_x}$).
To define the neighborhood of $\IG_x$ we use a method that compares the so-called \emph{baselines} of the text lines.
This is a common measure to evaluate the performance of a layout analysis result \cite{DBLP:journals/corr/GruningLDKF17}.
We call the tuple of two-dimensional points $B=\brac{B_1,\ldots,B_{\bracabs{B}}}\in\IP:=\brac{\IN^2}^*$a baseline.
We define
\begin{align*}
\FP^\IH=\brac{\FP_1^\IH,\ldots,\FP_N^\IH}\in\IP^N
\end{align*}
as tuple of polygons corresponding to $\IH$ and let $\FP^\IG$ be defined in the same manner for $\IG$.
From \cite[Section III A. 3)]{DBLP:journals/corr/GruningLDKF17} we use the \emph{Coverage Function} $COV:\IP\times\IP\times\IR\to\braceckig{0,1}\subset\IR$, which calculates the overlapping between two baselines for a given tolerance value.
The tolerance value $t:\IP^*\times \IN \to \IR$ is dependent on the geometric position of all ground truth baselines and the index of the ground truth baseline of interest (cf. \cite[Section III A. 2)]{DBLP:journals/corr/GruningLDKF17}).

We set
\begin{align*}
\NN\brac{\IG_x}:=\bracround{\IH_y\in\IH \mid COV\brac{\FP_y^\IH,\FP^\IG_x,t\brac{\FP^\IG,x}}>0.0},
\end{align*}
which implicitly restricts the set of valid assignment matrices in \eqref{E:A_vanilla}.
Indeed, setting baselines to be close if they have any connection is probably a very soft restriction, but reasonable to avoid erroneously non-assignments for close $\IH_y$ and $\IG_x$.
We modify \eqref{E:D_vanilla} with $i=\fl_{\fh}^y$ and $j=\fl_{\fg}^x$ by
\begin{align}
\label{E:D_geom}
\hspace{-0.2cm}\D{i}{j}^{R,G}=\min\hspace{-0.1cm}\left\{\hspace{-0.2cm}\begin{array}{ll}
\D{i-1}{j-1}^{R,G}&\hspace{-0.6cm}\text{if } \red{\IH_{y-1}\in\NN\brac{\IG_{x-1}}}\\
\D{\fl_\fh^{y-1}}{j}^{R,G}+\bracabs{\IH_{y-1}} &\text {if } y\geq 2\\
\D{i}{\fl_\fg^{x-1}}^{R,G}+\bracabs{\IG_{x-1}} &\text {if } x\geq 2
\end{array}\hspace{-0.2cm}\right\}.
\end{align}
\begin{theorem}[Minimal $LD^{R,G}$ calculation]
Let $\fh=f\brac{\IH}$ and $\fg=f\brac{\IG}$ the flattened sequences. The equation 
\begin{align*}
LD^{R,G}\brac{\IH,\IG}=\D{\bracabs{\fh}}{\bracabs{\fg}}^{R,G}
\end{align*}
holds.
\end{theorem}
\begin{proof}
	We use Theorem~\ref{theo:LD-vanilla} to prove that the additional constrained described in \eqref{E:A-geom} are fulfilled by the changes between \eqref{E:D_vanilla} and \eqref{E:D_geom}.
	Let  $A_{y,x}=1$, then $\IH_y\in\NN\brac{\IG_x}$ have to be shown. From $A_{y,x}=1$ it follows $\brac{y,x}\in W$.
	But in the proof of Theorem~\ref{theo:LD-vanilla} it is shown that $\brac{y,x}\in W$ can only be achieved if in \eqref{E:D_geom} (and \eqref{E:D_vanilla} the minimum is reached in the first row. This is only possible, if $\IH_y\in\NN\brac{\IG_x}$.
\end{proof}
	

\subsection{Non-Penalizing Segmentation Error}
\label{S:Implementation:SEG}
If we allow $\Phi\in\Psi$ to be applied to $\IH$, we have to modify the LD calculation at some positions.
As argued in Section \ref{S:Measure} we allow to map $\lbsymb$ to $\visiblespace$ without costs and vice versa.
We define $\bfl_\fh$ as expansion of $\fl_\fh$ by also containing the positions of the space character $\visiblespace\in\A$. We modify \eqref{E:d_vanilla} by
\begin{align}
\label{E:d_phi}
\d{i}{j}^{R,S}:=\begin{cases}
0 &\text{ if } \fh_i=\fg_j\\
1 &\text{ if } \fh_i\neq\fg_j \land \fh_i,\fg_j\in\A\red{\setminus\bracround{\visiblespace}}\\
\red{0} &\red{\text{ if } \fh_i=\lbsymb\land\fg_j=\visiblespace}\\
\infty &\text {else}
\end{cases}
\end{align}
and \eqref{E:D_vanilla} in points $\brac{i,j}$ with $i=\bfl_\fh^y$ and $j=\fl_\fg^x$  by
\begin{align}
\label{E:D_phi}
\D{i}{j}^{R,S}=\min\left\{\hspace{-0.2cm}\begin{array}{ll}
\D{i-1}{j-1}^{R,S}\\
\D{\bfl_\fh^{y-1}}{j}^{R,S}+\red{\bfl_\fh^y-\bfl_\fh^{y-1}-1} &\text {if } y\geq 2\\
\D{i}{\fl_\fg^{x-1}}^{R,S}+\bracabs{\IG_{x-1}} &\text {if } x\geq 2
\end{array}\hspace{-0.2cm}\right\},
\end{align}
which also allows to skip single words. This leads to the updated Algorithm \ref{alg:splitbestpathPhi}, which implicitly returns the best segmentation $\IH':=\Phi\brac{\IH}$.
\begin{algorithm}
	\small
	\RestyleAlgo{boxed}
	\DontPrintSemicolon
	\SetKwInOut{Input}{input}
	\SetKwInOut{Output}{output}
	\SetKwInOut{ret}{return}
	\SetKwComment{Comment}{\%}{}
	\Input{$P,\fl_\fg,\fh$}
	\Output{$U,V,W\red{,\IH'}$}
	\BlankLine	\BlankLine
	$U,V,W\gets \emptyset$\;
	\red{$\IH'\gets \braceckig{\;}$}\;
	$\fp\gets P_2$\Comment*{$P_1=\brac{0,0}$ is not of interest}
	\For{$i=3,\ldots,\bracabs{P}$}{
		$\fq\gets P_i$\;
		\If{$\fq_2\in\fl_\fg$}{
			\Comment{Found line break in $\fg$ ($\fg_{\fq_2}=\lbsymb$)}
			$x\gets index\brac{\fl_\fg;\fp_2}$\Comment*{$x$-th $\lbsymb$ in $\fg$}
			$x'\gets index\brac{\fl_\fg;\fq_2}$\Comment*{$x'$-th $\lbsymb$ in $\fg$}
			\If{$\fp_1<\fq_1$}{
				\red{$\fh'\gets \fh_{\fp_1+1:\fq_1-1}$}\;
				\red{$\fh'\gets replace\brac{\fh';\lbsymb;\visiblespace}$}\;
				\red{$\IH'.append\brac{\fh'}$}\;
				\If{$x<x'$}{
					$W\gets W\cup\bracround{\brac{\red{\bracabs{\IH'}},x}}$\Comment*[r]{$\IH_y$ maps $\IG_x$}
				}\Else{
					$U\gets U\cup\bracround{\red{\bracabs{\IH'}}}$\Comment*[r]{delete $\IH_y$}
				}
			}\Else{
				$V\gets V\cup\bracround{x}$\Comment*[r]{delete $\IG_x$}
			}	
			$\fp\gets \fq$\Comment*{end point is the new start}
		}
	}
	\Return{$U,V,W\red{,\IH'}$}
	\caption{\label{alg:splitbestpathPhi}SplitBestPathWithSegmentation}
\end{algorithm}
\begin{theorem}[Minimal $LD^{R,S}$ calculation]
	Let
	\begin{align*}
	\Phi^*=\argmin\limits_{\Phi\in\Psi}LD^R\brac{\Phi\brac{\IH},\IG}
	\end{align*}
	be the best partition minimizing \eqref{E:LD-seg}.
	For the LD calculated by \eqref{E:d_phi} and \eqref{E:D_phi} 
	\begin{align*}
	LD^R\brac{\IH^*,\IG}=LD^{R,S}\brac{\IH,\IG}:=\D{\bracabs{\fh}}{\bracabs{\fg}}^{R,S}
	\end{align*}
	holds. Algorithm \ref{alg:splitbestpathPhi} returns the best partition $\IH^*=\Phi^*\brac{\IH}$.
\end{theorem}
\begin{proof}
Clearly, the inequality $LD^R\brac{\IH^*,\IG}\leq LD^{R,S}\brac{\IH,\IG}$ holds due to the optimality of $\IH^*$.

To show $LD^R\brac{\IH^*,\IG}\geq LD^{R,S}\brac{\IH,\IG}$, let $\fh^*=f\brac{\IH^*}$ and $\fh=f\brac{\IH}$ be the flattened sequences.
Let $P^*$ be the best path of $LD^R\brac{\IH^*,\IG}$.
We show that $P^*$ is also a path in $LD^{R,S}\brac{\IH,\IG}$ with the same cost.
Since $\fh^*$ and $\fh$ can only differ in $i\in\bfl_\fh=\bfl_{\fh^*}$, we only have to show that \eqref{E:D_vanilla} is equal to \eqref{E:D_phi} in points  $\brac{i,j}\in P^*$ with $i=\bfl_\fh^y$.

For $j=\fl_\fg^x$, the equations only differ in the path which deletes $\IH^*_{y'}$ with $i=\fl_{\fh^*}^{y'}=\bfl_{\fh^*}^{y}$.
Because the minimal $\bracabs{\IH_{y'}^*}$ is achieved if $\IH_{y'}^*$ contains no spaces, we know $\bracabs{\IH_{y'}^*}=\bfl_\fh^y-\bfl_\fh^{y-1}-1$, so for $j=\fl_\fg^x$ the equations are equal.

For $j\notin \fl_\fg$, \eqref{E:d_vanilla} and \eqref{E:d_phi} only differ in $\visiblespace=\fg_j$, but for both possible values $\fh_i\in\bracround{\visiblespace,\lbsymb}$ we get $\d{i}{j}^R=\d{i}{j}^{R,S}$, so they are equal.

From $LD^R\brac{\IH^*,\IG}\geq LD^{R,S}\brac{\IH,\IG}$ and $LD^R\brac{\IH^*,\IG}\leq LD^{R,S}\brac{\IH,\IG}$ it follows equality.
\end{proof}
\subsection{Accepting Reading Order Errors}
\label{S:Implementation:RO}
Since the number of possible permutations of the text lines is too large, to exactly calculate the minimal LD, a heuristic will be defined to find the best map between $\IH$ and $\IG$.
Therefore, \eqref{E:D_vanilla} is changed at positions $i=\fl_\fh^y$ and $j=\fl_\fg^x$ by allowing to `jump` between hypothesis lines:
\begin{align}
\label{E:D_pi}
\D{i}{j}=\min\left.\begin{cases}
\D{i-1}{j-1}\\
\min\limits_{ i'\in\fl_\fh\setminus\bracround{i}}\D{i'}{j}
\end{cases}\right\}\quad
\end{align}
This allows the algorithm to find the optimal $\IH_y$ for each $\IG_x$.
Due to these jumps Alg. \ref{alg:splitbestpath} can now return tuples in $W$ having the same value in the first component.
This leads to $\left\|A\right\|_\infty>1$ and $A\notin\AA$. The idea for the greedy Alg. \ref{alg:greedyLDcalculation} is to assign $\IH_x$ to $\IG_y$, which minimizes
\begin{align*}
\argmin_{\IH_y\in\IH} CER\brac{\IH_y,\IG_x}
\end{align*}
Thus, the algorithm ``locally'' finds the minimal CER for each $\IG_x$.
The HYP line with higher CER stays in the set of unmatched lines as well as GT lines, that where not mapped by Alg. \ref{alg:greedyLDcalculation}.
On these subsets the algorithm is applied recursively.
The number of recursive calls is bounded by $\bracabs{\IG}$, because \eqref{E:D_pi} does not allow to skip $\IG_x$ and at least the first component of $L$ in Alg. \ref{alg:greedyLDcalculation}, Line 7 leads to a reduction of $\IH$ and $\IG$. In practice, the recursion depth is between $1$ and $4$, whereas $\IG$ is reduced very fast over the depth.
\begin{algorithm}
	\small
	\DontPrintSemicolon
	\SetKwInOut{Input}{input}\SetKwInOut{Output}{output}
	\SetKwComment{Comment}{\%}{}
	\Input{HYP: $\IH$}
	\Input{GT: $\IG$}
	\Output{greedy minimal LD: $LD\brac{\IH,\IG}$}
	\If{$\IG=\emptyset$}{
		\Return{$\sum\limits_{\IH_y\in\IH}\bracabs{\IH_y}$}\;
	}
	\If{$\IH=\emptyset$}{
		\Return{$\sum\limits_{\IG_x\in\IG}\bracabs{\IG_x}$}\;
	}
	$P\gets$ runDynProg($\fh,\fg$)\;
	$U,V,W\gets$ SplitBestPath($P,\fl_\fh,\fl_\fg$)\Comment*[r]{see Alg. \ref{alg:splitbestpath}}
	$L\gets\braceckig{sort\brac{W;(y,x):: CER\brac{\IH_y,\IG_x}}}$\;
	\Comment{returns array with entries (y,x) sorted by CER}
	$D\gets 0$\;
	\For{$k\leftarrow 1$ \KwTo $\bracabs{L}$}{
		$(y,x)\gets L\braceckig{k}$\;
		\If{$\IH_y\in\IH$}{
			$\IH\gets\IH\setminus\bracround{\IH_y}$\;
			$\IG\gets\IG\setminus\bracround{\IG_x}$\;
			$D\gets D + LD\brac{\IH_y,\IG_x}$\;
		}
	}
	\Return{$D+greedy\_LD\brac{\IH,\IG}$}
	\caption{\label{alg:greedyLDcalculation}greedy\_LD}
\end{algorithm}
\section{Conclusion and Future Works}
\label{S:Conclusion}
We have introduced a measure to evaluate an end-to-end text recognition system. Dependent on its configuration it considers the reading order, segmentation errors and the geometric position. So it closes the gap between a raw character error rate (which so far was only properly defined on text line level) and bag-of-word (which is a retrieval measure on words, that mostly takes the geometric position into account).

Further research can be done to close the gap towards key word spotting (KWS) measures like \emph{ mean average precision} (mAP) or \emph{general average precision} (gAP).
\section*{Acknowledgment}
\ifthenelse{\blind}{BLIND ACKNOWLEDGMENT}{
This work was partially funded by the European Union's Horizon 2020 research and innovation programme under grant agreement No 674943 (READ -- Recognition and Enrichment of Archival Documents). }
\newpage

\bibliographystyle{IEEEtran}

\bibliography{lit}

\end{document}